\documentclass[11pt,letterpaper]{article}
\usepackage{graphicx}
\usepackage{subfigure}
\usepackage{paralist}
\usepackage{natbib}
\usepackage{algorithm}
\usepackage{algorithmic}
\usepackage{amsmath,amssymb,xspace,amsthm,Tabbing,bm}
\usepackage{color,mathtools}
\usepackage{listings}
\usepackage{multirow}
\usepackage{fullpage}

\def \fro {\mathrm{F}}
\def\xray{\textsc{Xray} }
\def\hott{\emph{Hottopixx} }
\newtheorem{thm}{Theorem}[section]
\newtheorem{lem}{Lemma}[section]

\newtheorem{defn}{Definition}[section]
\def\argmax{\mathop{\rm arg\,max}}
\def\argmin{\mathop{\rm arg\,min}}

\newcommand{\reals}{\mathbb{R}}

\newcommand{\cc}{\mathcal{C}}

\def\argmax{\mathop{\rm arg\,max}}
\def\argmin{\mathop{\rm arg\,min}}

\newcommand{\vX}{\mathbf{X}}
\newcommand{\vY}{\mathbf{Y}}
\newcommand{\vA}{\mathbf{A}}
\newcommand{\vB}{\mathbf{B}}

\newcommand{\vU}{\mathbf{U}}

\newcommand{\vC}{\mathbf{C}}

\newcommand{\vN}{\mathbf{N}}
\newcommand{\vQ}{\mathbf{Q}}
\newcommand{\vR}{\mathbf{R}}
\newcommand{\vS}{\mathbf{S}}

\newcommand{\vW}{\mathbf{W}}
\newcommand{\vH}{\mathbf{H}}

\newcommand{\vLam}{\bm{\Lambda}}

\newcommand{\vx}{\mathbf{x}}

\newcommand{\vdelta}{\boldsymbol{\delta}}

\newcommand{\vs}{\boldsymbol{s}}

\newcommand{\vh}{\boldsymbol{h}}
\newcommand{\ve}{\mathbf{e}}

\newcommand{\vz}{\boldsymbol{z}}

\newcommand{\vb}{\mathbf{b}}
\newcommand{\vv}{\mathbf{v}}
\newcommand{\vr}{\mathbf{r}}

\newcommand{\vu}{\mathbf{u}}
\newcommand{\vp}{\mathbf{p}}
\newcommand{\vq}{\mathbf{q}}

\newcommand{\bq}{\begin{equation}}
\newcommand{\eq}{\end{equation}}
\newcommand{\ba}{\begin{eqnarray}}
\newcommand{\ea}{\end{eqnarray}}

\def\R{{\reals}}

\newcommand{\mbf}[1]{\mathbf{#1}}

\newcommand{\mcal}[1]{\mathcal{#1}}
\newcommand{\remove}[1]{}

\newcommand{\red}[1]{{\color{red} #1}}

\newcommand{\Cpp}{C\kern-0.05em\texttt{+\kern-0.03em+}}

\newcommand{\ConceptCpp}{ConceptC\kern-0.05em\texttt{+\kern-0.03em+}}
	
\title{Fast Conical Hull Algorithms for Near-separable Non-negative Matrix Factorization}
\author{\hspace{2mm}{Abhishek Kumar}\thanks{This work was done when AK was a summer intern at IBM Research.} $^\dagger$\hspace{2.4cm} {Vikas Sindhwani}$^\ddagger$\hspace{2.4cm} {Prabhanjan Kambadur}$^\ddagger$  \\\hspace{-4mm}{abhishek@cs.umd.edu}\hspace{1.8cm} {vsindhw@us.ibm.com}\hspace{2cm} {pkambadu@us.ibm.com} \\ \\ $^\dagger$Dept. of Computer Science, University of Maryland, College Park, MD, USA \\ $^\ddagger$IBM T.J. Watson Research Center, Yorktown Heights, NY 10598 USA
            }
\date{}

\begin{document}

\maketitle
\begin{abstract} 
The separability assumption~\citep{DonohoStodden,arora.stoc12} turns
non-negative matrix factorization (NMF) into a tractable problem. Recently, a
new class of provably-correct NMF algorithms have emerged under this
assumption. In this paper, we reformulate the separable NMF problem as that of
finding the extreme rays of the conical hull of a finite set of vectors. From
this geometric perspective, we derive new separable NMF algorithms that are
highly scalable and empirically noise robust, and have several other favorable
properties in relation to existing methods. A parallel implementation of our
algorithm demonstrates high scalability on shared- and distributed-memory machines.
 

\end{abstract}

\section{Introduction}
A data matrix $\vX$ of size  $m\times n$ is said to admit a Non-negative Matrix Factorization (NMF) 
with inner-dimension $r$, if $\vX$ can be expressed as $\vX = \vW \vH$ where $\vW, \vH$ are two non-negative matrices of 
dimensions $m\times r$ and $r \times n$  respectively. In many applications, a compact (i.e., small $r$) approximate NMF tends to provide a natural and interpretable part-based decomposition of the data~\citep{LeeSeung}, often more appealing than other low-rank factorizations. NMFs arise  pervasively in a variety of signal separation problems, such as modeling topics in text and hyperspectral image analysis~\citep{nmf:book}.

\begin{figure}[t]
\begin{center}
\includegraphics[clip=true,trim=0 1cm 0 2cm,height=2.8cm, width=0.5\linewidth]{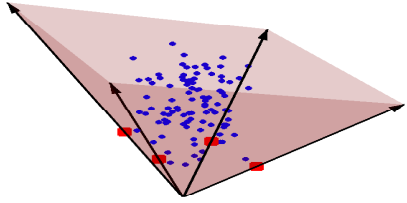}
\end{center}
\label{fig:geom}
\vspace{-0.4cm}
\caption{{\it Geometry of the NMF Problem. Separability implies that data is contained in a cone generated by a subset of $r$ ``anchor" points (red squares).}}
\vspace{-0.3cm}
\end{figure}

Figure~\ref{fig:geom} shows the geometry of the NMF problem. As a point cloud in $\reals^m$, all the $n$ columns of $\vX$ are contained inside a cone that is generated by $r$ non-negative vectors in $\reals^m$ comprising the columns of $\vW$. For any matrix $\vA$, let $cone(\vA)$ denote the set obtained by taking linear combinations of the columns of $\vA$ with non-negative coefficients. Then, the goal is to find a non-negative matrix $\vW$, with just $r$ columns, such that: 
$cone(\vX) \subseteq cone(\vW) \subseteq \reals_+^m$, where $\reals^m_+$ denotes the non-negative orthant in $\reals^m$. Such polyhedral nesting problems studied in computational geometry are known to be NP-hard, which makes the exact and approximate NMF problem also NP-hard~\citep{Vavasis.09}.
Faced with such results, almost the entire algorithmic focus in the NMF literature, e.g.,~\citep{nmf:book,LeeSeung,cjlin.nmf07,dhillon.kdd11}, has centered on treating the problem as an instance of general non-convex programming, leading to heuristic procedures that lack optimality guarantees beyond convergence to a stationary point of the objective function for approximate NMF. Very recently, in a series of elegant papers~\citep{arora.stoc12,bittorf.12,gillis.12,esser.11,elhamifar.cvpr12}, promising alternative approaches have been developed based on certain {\it separability} assumption on the data which enables the NMF problem to be solved exactly. Geometrically, the assumption states the following: \\ \vspace{-4mm}
\begin{defn}{\bf Separability Assumption:} {\it The entire dataset, i.e. all columns of $\vX$, reside in a cone generated by a small subset of $r$ columns of  $\vX$.}\vspace{-1mm}\end{defn} 
In algebraic terms, $\vX = \vW \vH = \vX_A \vH$ so that the $r$ columns of $\vW$ are hidden among the columns of $\vX$ (indexed by an unknown subset of indices $A$). Equivalently, a corresponding subset of $r$ columns of $\vH$ happen to constitute the $r\times r$ identity matrix. We refer to these columns as {\it anchors}~\citep{arora.stoc12}. Informally, in the context of topic modeling problems where $\vX$ is a document-word matrix and $\vW,\vH$ are document-topic and topic-term associations respectively, the separability assumption equivalently posits the existence of special {\it anchor words} in the vocabulary, whose occurence uniquely identifies the presence of a topic, and whose usage across the corpus is collectively predictive of the usage of all the other words.  The separability assumption was investigated earlier by~\citet{DonohoStodden} who showed that it implied uniqueness of the NMF solution, modulo permutation and scaling. In order to place our contributions in the right context, we first briefly provide a flavor of recently proposed separable NMF algorithms. \\

{\bf Related Work}: Assuming that the columns of $\vX$ are normalized to have unit $l_1$-norm, the separable NMF problem reduces to that of finding the extreme points (that is, points inexpressible as convex combinations of other points) of the convex hull of the columns~\citep{arora.stoc12}. A Linear Program (LP) can be setup to attempt to express a given column as a convex combination of the other columns. If this LP declares infeasibility, an extreme point is identified. This approach~\citep[Section 5]{arora.stoc12} requires solving $n$ feasibility LP's each involving $n-1$ variables which is not scalable for many problems of interest. A noise-robust version of the procedure further requires knowledge of parameters that are hard to estimate apriori.~\citet{bittorf.12} formulate a single LP whose solution resolves the exactly separable NMF problem. An extension is also developed for noise-robustness. Instead of invoking a general LP solver, a specialized algorithm is derived based on an incremental stochastic gradient descent procedure, and its parallel (multithreaded) implementation is benchmarked on large datasets. On the other hand, this algorithm requires estimates of primal and dual step sizes, converges only asymptotically, and does not explicitly exploit the sparsity of the final solution. \citet{gillis.12} develop a highly scalable approach closely related to rank-revealing QR factorizations for column subset selection. A perturbation analysis of this algorithm under noise is also presented. In~\citet{esser.11}, column subset selection is cast essentially as a form of multivariate regression with row-sparsity inducing norms, e.g., see~\citet{mahoney.cur}. Algorithms derived in this framework are asymptotically convergent, and sensitive to near-duplicate columns, making it necessary to perform certain adhoc preprocessing steps. \\

{\bf Contributions}: We present a new family of highly scalable and empirically noise-robust algorithms for separable NMFs, with several favorable properties:
\begin{compactitem}[$\circ$]
\item The algorithms produce a correct solution for the separable case after exactly $r$ iterations. They require no additional parameters. They are closely related to convex and conical hull finding procedures proposed in the computational geometry literature~\citep{clarkson,dula}. Computationally, the algorithms bear some resemblence to simultaneous Orthogonal Matching Pursuit~\citep{SparseBook,tropp.somp06} for sparse greedy reconstruction of multiple target variables from the same  subset of input variables. We also derive a variant based on this connection that performs quite well under noise.
\item Under controlled noise conditions in synthetic datasets and on real-world topic modeling problems, our algorithms consistently outperform other separable NMF techniques with respect to multiple performance metrics. Our methods are highly competitive with existing non-convex NMF algorithms, but are free of sub-optimal local minima and associated initialization issues. 
\item The solution for $(r-1)$ target anchors is contained in the solution for $r$ target anchors (unlike non-convex NMF methods), which makes it easier to do model selection on real-world datasets by keeping track of performance on a validation set. 
\item The algorithms are highly scalable and have small memory footprint. The sparsity of the data, the intermediate variables and the final solution is carefully exploited in a high-performance parallel and distributed implementation which scales excellently on both shared- and distributed-memory machines. For example, a twitter corpus with 125-thousand tweets can be factorized for $r=100$ in less than 10 seconds on a commodity 8-core machine.  
 \item Unlike all existing algorithms, no column normalization is needed. Such normalization interferes with the TFIDF weightings routinely used in text modeling applications, leading to performance loss. 
 \item Unlike~\citet{esser.11}, the algorithms do not require any special preprocessing to eliminate duplicate or near-duplicate columns.
\end{compactitem} 

\section{Fast Conical Hull Algorithms}

{\bf An informal description}: Figure~\ref{fig:cone2} provides some geometric intuition underlying the proposed approach. The algorithm executes $r$ iterations. 
In each iteration a new anchor column is identified. This corresponds to expanding a cone one extreme ray at a time, until the entire dataset is eventually contained in the cone defined by the full set of anchors. Figure~\ref{fig:cone2} illustrates one step of the algorithm where there is an existing cone defined by three extreme rays (marked 1 to 3). To identify the next extreme ray, the algorithm picks a point outside the current cone (a green point) and projects it to the current cone to compute a residual vector (we call this the {\it projection step}). This residual vector separates the current cone from at least one non-selected extreme ray that can be found by maximizing a specific selection criteria (we call this the {\it detection step}). Intuitively, the algorithm picks a face of the current cone (spanned by rays 1 and 3 in Figure~\ref{fig:cone2}) that ``sees" exterior points and rotates this face towards the exterior until it hits the ``last'' point. In the example shown in Figure~\ref{fig:cone2}, ray 4 is identified as a new extreme ray. 
\begin{center}
\begin{figure}[h]
\centering
\includegraphics[clip=true,trim=0 2mm 0 2.5cm,height=4.15cm, width=0.5\linewidth]{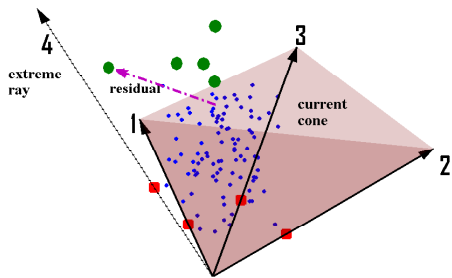}
\label{fig:cone2}
\vspace{-3mm}
\caption{Geometry of the Conical Hull algorithms}
\end{figure} 
\end{center}
These geometric intuitions are inspired by~\citet{clarkson,dula} who present LP-based algorithms for general convex and conical hull problems. Their algorithms are also directly applicable in our NMF setting, provided the data satisfies the separability assumption exactly. In this case, the residual of {\it any} single exterior point can be used to correctly expand the cone as described above. However, anchor detection criteria derived from {\it multiple} residuals demonstrates radically superior noise robustness, as we report in the experimental section.  The emphasis on scalability and noise-robustness thus leads us to a new family of algorithms whose implementation (and associated proof of correctness) is distinct from prior work.  \\

{\bf Cones, Extreme Rays and Projection}:  Here, we provide a short background on relevant geometric concepts and set some notation.  Recall that a {\it cone} $\cc$ is a non-empty convex set that is closed with respect to taking conic combinations (i.e., linear combinations with non-negative coefficients) of its elements.   A {\it ray} in $\cc$ generated by a vector $\vx\neq 0\in \cc$ is the set of all vectors $\{t\vx : t \geq 0\}$. A ray $R$ is an {\it extreme ray} if its generators cannot be expressed by taking conic combinations of elements in $\cc$ that do not themselves belong to $R$. A cone is called {\it finitely generated} if its elements are conic combinations of a finite set of vectors, and {\it pointed} if it does not contain both a vector $\vx$ as well as its negation $-\vx$.  A fundamental result (e.g., see~\cite{NemirovskiNotes}) states: {\it a pointed, finitely generated cone $\cc$ possesses a finite and unique set of extreme rays, and $\cc$ is the conical hull of the generators of these extreme rays}. Furthermore, the generators of these extreme rays are a subset of the finite set of vectors used to originally express the cone.   In the NMF context, note that any cone contained in $\reals^m_+$ is pointed. This implies that $cone(\vX)$ can also be described by a minimally compact set of generators, i.e., $cone(\vX)=cone(\vX_A)$ where $A$ uniquely indexes the extreme rays (anchors). Thus, a non-negative matrix $\vX$ admits a separable NMF with inner-dimension $r$ if the number of extreme rays of $cone(\vX)$, i.e. size of $A$, coincides with $r$. A face of a cone is the intersection between the cone and a supporting hyperplane. The projection of a point $\vx$ onto the cone generated by columns of a matrix $\vW$, i.e. computing $\vz^\star = \argmin_{\vz \in cone(\vW)} \|\vx - \vz\|^2_2$, can be obtained by solving a non-negative least squares problem, i.e., computing $\vh^\star = \argmin_{\vh\geq 0} \|\vx - \vW \vh\|^2_2$ and setting $\vz^\star = \vW\vh^\star$. All columns of $\vX$ can be simultaneously projected by solving $\vH^\star = \argmin_{\vH\geq 0}\|\vX - \vW \vH\|^2_2$. We will use the notation $\vR$ to denote the residual matrix after a projection operation, i.e., $\vR = \vX - \vW \vH^\star$. We will use the notation $\vX_i, \vR_i$ to denote the $i^{th}$ column of $\vX$ and its corresponding residual. The notation $\vq_+$ will denote the vector obtained by setting all negative entries of the vector $\vq$ to 0. \\


\subsection{Algorithm Description, Correctness Result and Variations}
Algorithm~\ref{alg:conic} details the steps of the proposed family of algorithms which we call \xray. Each iteration consists of two steps:
(i) a {\it detection step} that finds a column(s) of $\vX$ to be added as an anchor, and (ii) a {\it projection step} 
where all data points are projected onto the current cone to get the residuals. Projection is done by solving 
simultaneous nonnegative least squares problem using Algorithm~\ref{alg:nnls}. Every residual vector $\vR_i$ 
obtained after the projection step is normal to one of the faces of the current cone. In the selection step, we
pick a face of the current cone (identified by its normal $\vR_i$), normalize all the data points
to lie on the hyperplane $\vp^T\vx=1$ $\left(\vY_j = \frac{\vX_j}{\vp^T\vX_j}\right)$ for a strictly positive vector $\vp$, and expand the current cone by selecting 
an extreme ray that maximizes the inner product $\vR_i^T \vY_j$. The selection step can be implemented in various ways - some options are listed in Algorithm~\ref{alg:conic}.
\def\conic{}
\begin{center}
\begin{algorithm}[tb]
\caption{\xray: Algorithms for Separable NMF}
\label{alg:conic}
\begin{algorithmic}
{\small 
\STATE {\bfseries Input:} $\vX \in \reals_+^{m\times n}$ separable with inner dimension $r$ 
\STATE {\bfseries Output:} $\vW\in \reals_{m\times r}, \vH \in \reals_{r\times n}$, $r$ indices in $A$ \\
~~~~~~~~~~~~such that: $\vX = \vW \vH$, $\vW = \vX_A$ 
\STATE {\bfseries Initialize:} $\vR \gets \vX$, $A\gets \{\}$
\WHILE{$|A|<r$} 
\STATE 1. {\it {\bfseries Detection Step}: {\it Find an extreme ray}}.
\STATE  Below, $\vp$ is a strictly positive vector (not collinear with $\vR_i$):
\begin{equation}
j^\star = \argmax_j \frac{\vR_i^T \vX_j}{\vp^T \vX_j}~~\textrm{for any}~~i:\|\vR_i\|_2>0
~\label{eq:selection}
\end{equation}
Some exterior point selection criteria:
\begin{eqnarray}
\textrm{\textit{rand}}:  & ~~\textrm{any random}~~i : \|\vR_i\|_2>0\label{eq:rand}\\
\textrm{\textit{max}}:   & i = \argmax_k \|\vR_k\|_2\label{eq:max}\\
\textrm{\textit{dist}}:  & ~~i = \argmax_k \|\left(\vR_k^T \vX\right)_+\|_2\label{eq:dist} \\
\textrm{A \textit{greedy} variant:}  & ~j^* = \argmax_j \frac{\|(\vR^T \vX_j)_+\|_2^2}{\|\vX_j\|^2_2}\label{eq:greedy}
\end{eqnarray}
\vspace{-3mm}
\STATE 2. Update: $A\gets A\cup \{j^*\}$ (see Remarks)

\STATE 3. {\it {\bfseries Projection Step}: Project onto current cone.} 
\begin{equation}\vH = \argmin_{\vB\geq 0}\|\vX - \vX_A \vB\|^2_2~~\textrm{(Algorithm 2)}\label{eq:nnls}\end{equation}
\STATE 4. Update Residuals (not explicitly): $\vR = \vX - \vX_A \vH$ 
\ENDWHILE
}
\end{algorithmic}
\end{algorithm}
\end{center}

\def\nnls{}
\begin{algorithm}[tb]
\caption{\nnls Solver for: $\argmin_{\vB\geq 0}\|\vX - \vX_A \vB\|^2_2$}
\label{alg:nnls}
\begin{algorithmic}
{\small 
\STATE {\bfseries Input:} $\vX \in \R^{m \times n}$, Index set $A$ with $r$ indices, Initial value for warm-starts:  $\vB_{init}\in \R^{r \times n}$\\
Convergence paramaters $tol$, $maxcycles$
\STATE {\bf Initialize:} $\vB = \vB_{init}$
\STATE Set: $\vC = \vX^T \vX \in \reals^{n \times n}$
\STATE $\vS = \vC_{A,A}\in \reals^{r \times r}$
\STATE $\vs = diag(\vS)$
\STATE $\vU = \vB^T \vS \in \reals^{n\times r}$
\WHILE{true}
\FOR{$i=1$ \ldots $r$ ({\it //cyclic coordinate descent})} 
\STATE $\vb = \vdelta = (\vB^T)_i$
\STATE $\vu = \vU_i - s_i\vb$
\STATE $\vb = s^{-1}_i(\vC_i - \vu)_+$
\STATE $\delta = \vb  - \vdelta$
\STATE $\vU = \vU +  \vdelta\vS_i^T$~~~{\it // sparse rank-1 update}
\STATE $(\vB^T)_i = \vb$
\ENDFOR
\STATE $objective = \|\vX\|^2 + \sum_{i=1}^r (\vU_i + \vC_i)^T (\vB^T)_i  $
\STATE Exit if $\Delta objective < tol$, or $\#$iters $ > maxcycles$.
\ENDWHILE
}
\end{algorithmic}
\end{algorithm}
\vspace{-4mm}

\remove{
We propose the following three variants of the approach depending on the selection operator $\mcal{S}$ and the vector $p$ used in the
selection step:
\begin{itemize}
\item {\bf Conic:} $\mcal{S}(\vv) = \ve_i^T \vv$ for any $i$ such that $\vR_i\neq 0$, and $\vp$ is any positive 
vector such that $\vp$ and $\vR_i$ are linearly independent.
\item {\bf NNOMP:} $\mcal{S}(\vv) = \lVert (v)_+\rVert_2^2$ and $\vp = \frac{\vX_j}{\lvert \vX_j\rVert}$. 
\item {\bf NNOMP-sep:} $\mcal{S}(\vv) = \lVert (v)_+\rVert_2^2$ and $\vp$ is any positive vector. 
\end{itemize}
}

To show that \xray correctly identifies all the extreme rays, we need the following lemmas.
\begin{lem}
The residual matrix $\vR$, obtained after projection of columns of $\vX$ onto the current cone 
satisfies $\vR^T \vX_A \leq 0$, where $\vX_A$ are the extreme rays of the current cone.
\label{lem:kktnnls}
\end{lem}
\begin{proof}
Residuals are given by $\vR = \vX - \vX_A \vH$, where $\vH = \argmin_{\vB\geq 0} \lVert \vX - \vX_A \vB\rVert_\fro^2$. \\
Forming the Lagrangian for Eq.~\ref{eq:nnls}, we get
$L(\vB,\vLam) = \lVert \vX - \vX_A \vB\rVert_\fro^2 - tr(\vLam^T\vB)$, where the matrix $\vLam$ contains
the nonnegative Lagrange multipliers.
Differentiating w.r.t. $\vB$ and evaluating at the optimum $\vB = \vH$, we have the following from the KKT conditions:
$ \qquad 2\vX_A^T(\vX_A\vH - \vX) - \vLam = 0$ \\
$\quad \Rightarrow  -2\vX_A^T \vR = \vLam \geq 0 \quad \Rightarrow \vR^T \vX_A \leq 0$
\end{proof}
\begin{lem}
For any point $\vX_i$ exterior to the current cone, we have $\vR_i^T \vX_i > 0$, where $\vR_i$ is the residual
of $\vX_i$ obtained by projecting it onto the current cone. 
\label{lem:extcone}
\end{lem}
\begin{proof}
Let $\vR = \vX - \vX_A \vH$, where $\vH = \argmin_{\vB\geq 0} \lVert \vX - \vX_A \vB\rVert_\fro^2$ and $\vX_A$ 
are the extreme rays of the current cone.
From the KKT conditions (used in the proof of Lemma~\ref{lem:kktnnls}) we have $2\vR^T \vX_A = -\vLam^T$,
where $\vLam$ are the Lagrange multipliers.
Hence, $2\vR_i^T\vX_A = -\vLam_i^T$ ($i$th row of both left and right side matrices). From the complementary
slackness property, we have $\vLam_{ji}\vH_{ji} = 0\,\, \forall\, j,i$. Hence, $2\vR_i^T\vX_A\vH_i = -\vLam_i^T\vH_i = 0$. \\
Hence we have $\vR_i^T \vX_i = \vR_i^T (\vR_i + \vX_A \vH_i) = \lVert\vR_i\rVert_2^2 + \vR_i^T\vX_A\vH_i 
= \lVert\vR_i\rVert_2^2 > 0$ since $\vR_i\neq 0$. 
\end{proof}
Using the above two lemmas, we prove the following theorem regarding the correctness of 
Algorithm~\ref{alg:conic}.
\begin{thm}
The data point $\vX_{j^*}$ added at each iteration in the Detection step of Algorithm~\ref{alg:conic}, if the maximizer in Eqn.~\ref{eq:selection} is unique, is an extreme ray 
of $\cc$ that has not been selected in previous iterations. 
\end{thm}
\begin{proof}
Let the index set $A$ identify all the extreme rays of $\cc$. Under the separability assumption, we have $\vX = \vX_A \vH$. 
Let the index set $A^t$ identify the extreme rays of the current cone $\cc^t$. 

Let $\vY_j = \frac{\vX_j}{\vp^T\vX_j}$ and $\vY_A = \vX_A [diag(\vp^T\vX_A)]^{-1}$ (since $\vp$ is strictly positive, 
the inverse exists).
Hence $\vY_j =$ $\vY_A \frac{[diag(\vp^T\vX_A)] \vH_j}{\vp^T\vX_j}$. Let $\vC_j = \frac{[diag(\vp^T\vX_A)] \vH_j}{\vp^T\vX_j}$. 
We also have $\vp^T\vY_j = 1$ and $\vp^T\vY_A = \mbf{1}^T$. Hence, we have $1 = \vp^T\vY_j = \vp^T\vY_A \vC_j = \mbf{1}^T \vC_j$. 

Using Lemma~\ref{lem:kktnnls}, Lemma~\ref{lem:extcone} and the fact that $\vp$ is strictly positive, we have 
$\max_{1\leq j\leq n} \vR_i^T \vY_j = \max_{j\notin A^t} \vR_i^T \vY_j$. Indeed, for all $j\in A^t$ we have
$\vR_i^t \vY_j \leq 0$ using Lemma~\ref{lem:kktnnls} and there is at least one $j=i\notin A^t$ for which
$\vR_i^t \vY_j > 0$ using Lemma~\ref{lem:extcone}. Hence the maximum lies in the set $\{j: j \notin A^t\}$.

Further, we have $\max_{j\notin A^t} \vR_i^T \vY_j = \max_{j\notin A^t} \vR_i^T \vY_A \vC_j 
\leq \max_{j\in (A \setminus A^t)} \vR_i' \vY_j$. The second inequality is the result of the fact that
$\lVert \vC_j\rVert_1 =1$ and $\vC_j\geq 0$. This implies that if there is a unique maximum at a
$j^* = \argmax_{j\notin A^t}\vR_i^T \vY_j$, then $\vX_{j^*}$ is generator of an extreme ray of the cone $\cc$.
\end{proof}
{\bf Remarks}: (1) If the maximum occurs at two points $j_1^*$ and $j_2^*$, both these points $\vX_{j_1^*}$ and $\vX_{j_2^*}$
generate the extreme rays of the cone $\cc$. Hence both are added to anchor set $A$. 
If the maximum occurs at more than two points, 
some of these are the generators of the extreme rays of $\cc$ and others are conic combinations of these generators.
We can identify the extreme rays of this subset of points by calling Algorithm 1 recursively and add them to anchor set $A$. 
(2) Note that the algorithm is not influenced by presence of repeated anchors. (3) In the Algorithm, the vector $\vp$ simply needs to satisfy $\vp^T \vx_i >0, i=1\ldots n$. In our implementation, we simply used $\vp = [1,\ldots 1]\in \reals^m$, i.e., $\vp^T \vx_i = \|\vx_i\|_1$. 
(4) Note that unlike \citet{gillis.12}, we do not need $\vX_A$ to be full-rank. \\

{\bf Exterior Point Selection}: It can be noted that residual of any point exterior to the current cone (i.e., any $\vR_i\neq 0$)
can be used in the selection step of Algorithm~1. This gives us multiple ways of expanding
the current cone depending on which $i$ is chosen - all of which solve the separable problem but may behave very differently in the presence of noise.  Some natural options are listed in Algorithm 1: choosing a random exterior point (Eqn.~\ref{eq:rand}), one with maximum residual norm (Eqn.~\ref{eq:max}) or one which defines a normal to a supporting hyperplane of the current cone which ``sees"  maximum ``mass" of points in its positive halfspace, as measured by Eqn.~\ref{eq:dist}. In the experiments, we will refer to these variants as \xray (rand), \xray (max) and \xray (dist) respectively. \\

{\bf A Greedy variation for noisy data}:   In high dimensional noisy data almost all the points may masquerade as anchors.  
A natural choice is to expand the current cone greedily by selecting a point that best minimizes the current residual, 
i.e., $j^* = \argmin_j \min_{\vb>0} \lVert \vR - \vX_j \vb^T\rVert_F^2$. This selection criterion
simplifies to Eqn.\ref{eq:greedy} in Algorithm 1 (referred as \xray (greedy) henceforth). 
One may view this approach as implementing a nonnegative variant of simultaneous orthogonal matching  pursuit~\citep{tropp.somp06}, 
which is a greedy approach to the problem of sparse regression of multiple response variables on the same subset of explanatory variables, i.e., for solving $\min_{\vB\geq 0} \lVert \vX - \vX\vB\rVert_F^2$ s.t. $\lVert \vB\rVert_{0,1} = r$ where $\|\vB\|_{0,1}$ pseudo-norm counts the number of non-zero rows in $\vB$. 
In the context of separable NMF, both response variables and explanatory variables are the columns data matrix $\vX$. 
A relaxed version of this problem is solved in~\citet{esser.11} ($\min_{\vB\geq 0} \lVert \vX - \vX\vB\rVert_F^2 + \lambda\lVert\vB\rVert_{1,\infty}$). 
It is also possible to have $\lVert\vB\rVert_{1,2}$ penalized variant~\citep{tropp.06,mahoney.cur} which is natural for sparse multivariate regression problems.
Note that the greedy approach is not guaranteed to solve the separable NMF problem, but may perform well in the noisy settings as we observe in our experiments. Intuitively, this variant is concerned with greedily optimizing all residuals on average at every iteration, instead of making a decision based on the residual of a single, albeit well-chosen, exterior point. \\

{\bf Solution Refinement and Model Selection}: In practice, the separable solution $(\vW,\vH)$ as obtained from Algorithm 1 may be further refined with a few steps of alternating optimization with respect to a divergence measure of interest (e.g., Frobenius reconstruction $\|\vX - \vW\vH\|^2_{\fro}$). Also, in real-world datasets, the value of $r$ is typically unknown. Since our algorithms build the solution one anchor at a time, $r$ can be set based on a performance measure evaluated on held-out data. Alternatively, Algorithm 1 can exit if the amount of improvement from introducing a new anchor falls below a prespecified threshold.



\subsection{Scalability and Parallelization}
Here we describe various implementation details that allow us to gracefully scale to large sparse datasets (e.g., document-term matrices). The detection step can be parallelized by scoring the candidate anchors simultaneously. Likewise, the projection step involves solving Eqn.~\ref{eq:nnls}, which is separable in the columns of $\vB$ and hence can be optimized in parallel.  \\

{\bf Detection Step}: We avoid materializing the dense residual matrix $\vR$ in the evaluation of the anchor selection criteria. Instead, we score candidate anchors on-the-fly as we compute (but not explicitly materialize) a matrix $\vQ = \left(\vR^T \vX\right)_+ = \left(\vC - (\vC_A\vH)^T\right)_+$ where $\vC = \vX^T \vX$ denotes a covariance matrix (word-by-word for topic modeling applications). Here, the potential sparsity, symmetry of the covariance matrix $\vC$ as well as the non-negativity of $\vH$ can be further exploited. For example, if $\vC_{ij}=0$, the corresponding entry in the product $(\vC_A\vH)$ need not be computed, since the resulting negative value is anyway reset to zero by the $(\cdot)_+$ thresholding operator.  On a $P$ core machine, the selection criteria may be evaluated in $O(\frac{nnz(\vC) r}{P})$ time where $nnz(\vC)$ is the number of non-zeros in $\vC$. If $\vC$ is dense, we compute $\vQ$ using parallel dense BLAS-3 operations. The one time computation of $\vC$ is done via a parallel aggregation of rank-one outer-product terms defined by the rows of $\vX$. \\

{\bf Projection Step}: Algorithm 2 gives the steps of a cyclic block coordinate descent algorithm organized around very light-weight incremental sparsity-exploiting updates for solving Eqn.~\ref{eq:nnls} (derivation omitted for brevity).  The algorithm can be invoked in parallel on columns of $\vX$ to compute the corresponding columns of $\vB$. The previous value of $\vB$ is used to warm start the optimization and typically a very small number of iterations is needed for convergence.  

\vspace{2mm}
\section{Empirical Observations}
\label{sec:emp}
Here, we report extensive comparisons on synthetic and medium-scale topic modeling problems, and  benchmark our parallel implementation on large text datasets on multicore machines and distributed systems. We compare with the methods proposed in~\citet{bittorf.12} (abbrv. as \hott) and~\citet{gillis.12}
 (abbrv. as GV), as well as traditional NMFs based on alternating optimization \citep{nmf:book}. 
The source codes for \hott and GV were taken from the respective authors' websites.
In comparisons with~\citet{esser.11}, it was observed that it tends to select near-duplicate anchors, as also mentioned in~\citet{esser.11}. This characteristic causes it to consistently perform less favorably compared to other methods unless the data is preprocessed in an adhoc fashion to remove similar columns of $\vX$; 
hence we do not include it in our list of baselines. We also do not compare with~\citet{arora.stoc12} 
since \hott reportedly performs better~\citep{bittorf.12} and the algorithm requires parameters which are hard to guess apriori.

\remove{
\begin{figure}[t]
\centering
\includegraphics[width=6cm,height=6cm]{figures/swimmer4images.eps}  
\caption{Four randomly selected images from the Swimmer dataset~\cite{DonohoStodden}}
\label{fig:swimmer4images}
\end{figure}
}
\vspace{-1mm}
\subsection{Synthetic experiments}
\vspace{-1mm}
We perform a synthetic experiment that injects controlled amount of noise to corrupt the separable structure. 
Each entry of the matrix $\vW\in\mathbb{R}_+^{200\times 20}$ is generated i.i.d. according to a uniform distribution between 0 and 1.
The matrix $\vH\in\mathbb{R}_+^{20\times 210}$ is taken to be $[I_{20\times 20} \,\, {\vH}']$ where 
each column of ${\vH}'\in\mathbb{R}_+^{20\times 190}$ is generated according to a Dirichlet distribution 
whose parameters are chosen uniformly in $[0,1]$. The data matrix $\vX$ is set to $\vW\vH + \vN$ where each entry of
noise matrix $\vN$ is generated i.i.d. according to a Gaussian distribution with zero mean and std. dev. $\delta$.
Fig.~\ref{fig:gillisexp2} plots the fraction of correctly recovered anchors (averaged over 10 runs
for each value of $\delta$) against the noise level $\delta$ ranging from
$0$ to $1.5$. 
{\it The proposed \xray (max) shows the best noise-robustness in terms of anchor recovery, followed by \xray (dist) and GV}.
Although \xray (greedy) does not perfectly resolve the separable NMF problem ($\delta=0$), 
it performs better than \hott and is competitive with GV for near-separable case ($\delta >0$). 
As described below, on real datasets it turns out to be highly competitive.  
\xray (rand), although solves the separable problem ($\delta=0$), degrades significantly 
under noise, which shows that {\it proper selection of an exterior point to expand the current cone
is crucial for noise-robustness}.

\begin{figure}[h]
\centering
\includegraphics[clip=true,trim=0 0 0 7mm,width=7cm,height=5cm]{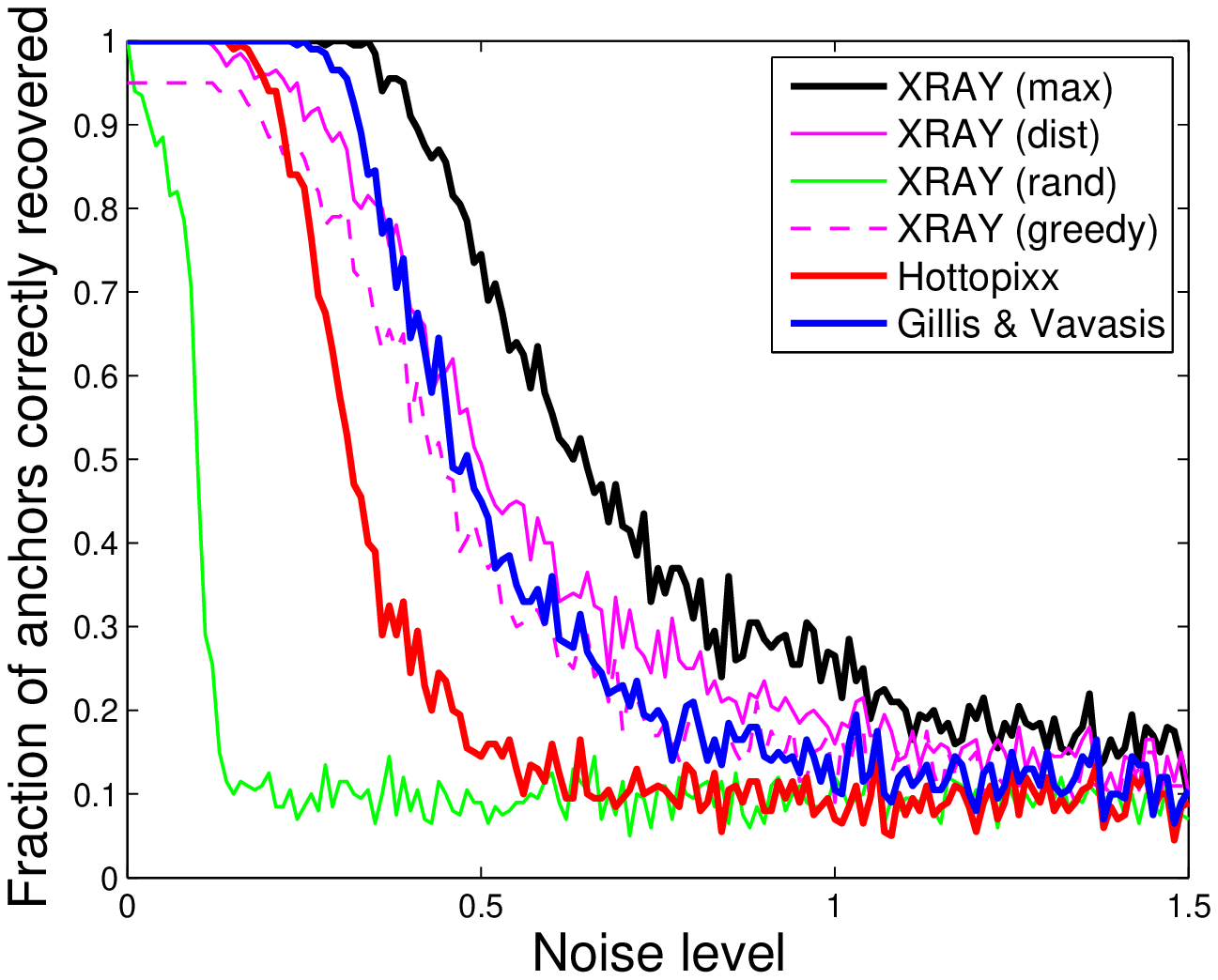}  
\caption{Anchor recovery rate versus noise level (best viewed in color)}
\label{fig:gillisexp2}
\end{figure}

\subsection{Medium-scale Topic modeling problems}
We evaluate the proposed methods on three human-labeled text datasets that are commonly used in topic modeling
literature: TDT-2 \citep{TDT2} ($m=9394, n=19528,r=30$), BBC \citep{bbc.greene06icml} ($m=2225,n=9635, r=5$) and
 Reuters \citep{reuters} ($m=7285, n=18221, r=10$). 
We used standard tf-idf representation with document frequency thresholding in constructing the data matrix $\vX$. 
As required in \hott and GV, we use $\ell_1$-normalized  
columns of $\vX$ (referred as matrix $\vX^{(\ell_1)}$ henceforth) to identify the anchor column indices $A^{(\ell_1)}$, 
and use the unnormalized data $\vX$ (and the corresponding
anchor columns $\vX_{A^{(\ell_1)}}$) for classification and clustering tasks. The use of $\vX_A^{(\ell_1)}$
in clustering and classification (for any index set $A$) resulted in significantly worse performance uniformly for
all methods so these results are not reported. For the sake of clarity in the figures, 
we do not show the results for \xray (max) which performed almost similar to \xray (dist) in these experiments.

\begin{figure*}[ht]
\newlength{\figwidthsvm}
\setlength{\figwidthsvm}{0.33\textwidth}
\begin{center}
\hspace{-5mm}\begin{minipage}{0.33\textwidth}
\centering
\includegraphics[width=1.1\columnwidth, height=4.5cm]{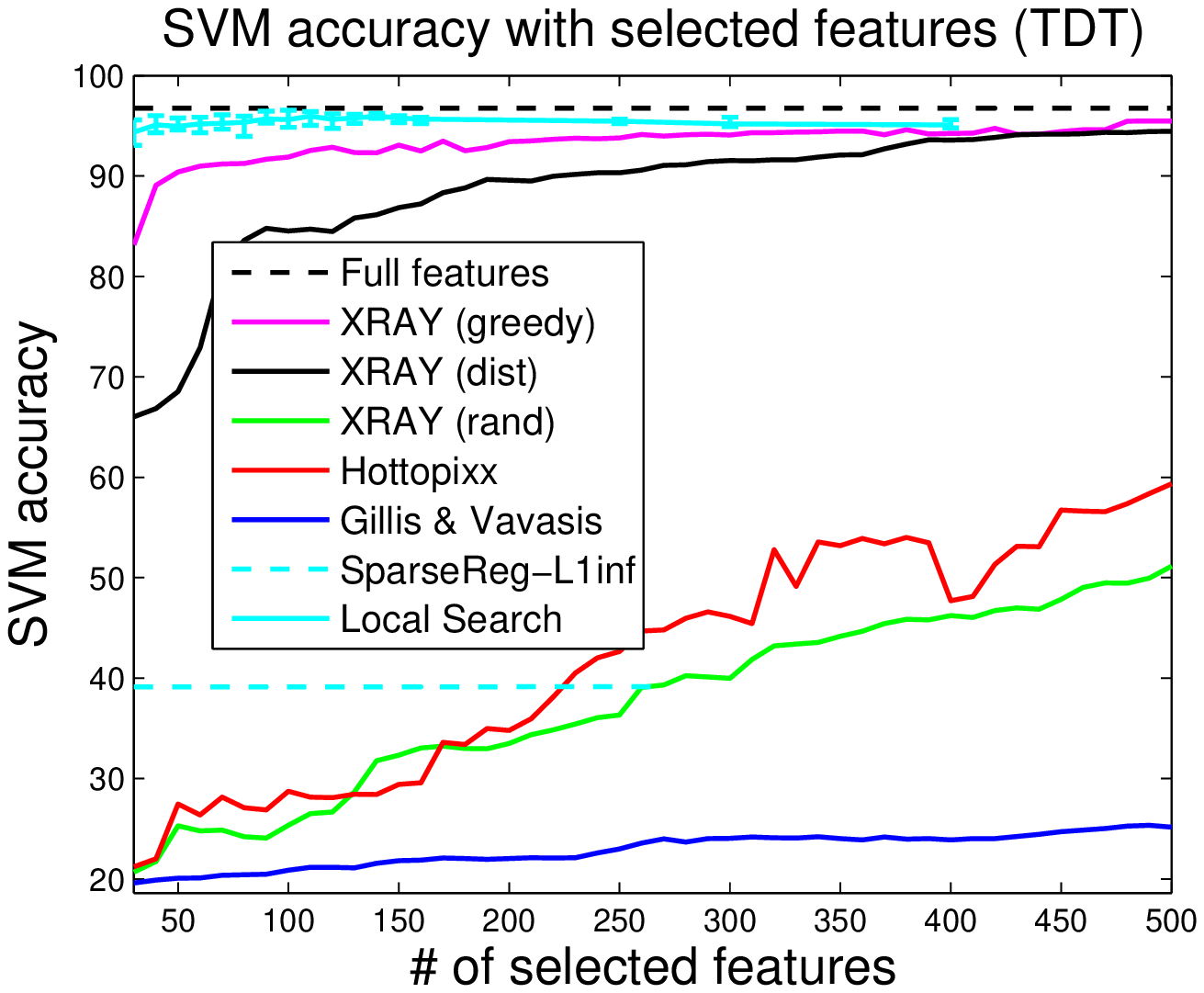}
\end{minipage}
\begin{minipage}{0.33\textwidth}
\centering
\includegraphics[width=1.1\columnwidth, height=4.5cm]{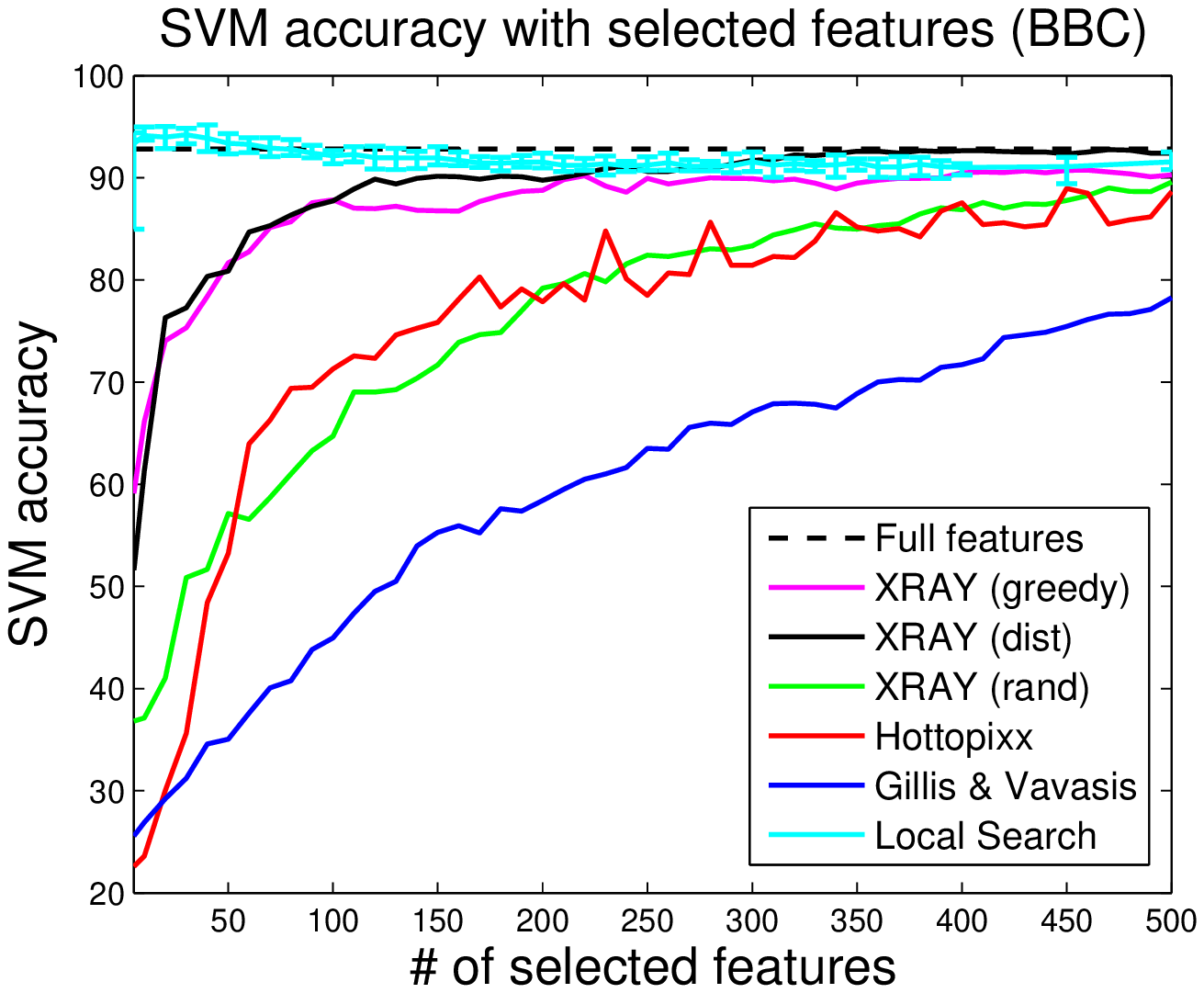}
\end{minipage}
\begin{minipage}{0.33\textwidth}
\centering
\includegraphics[width=1.1\columnwidth, height=4.5cm]{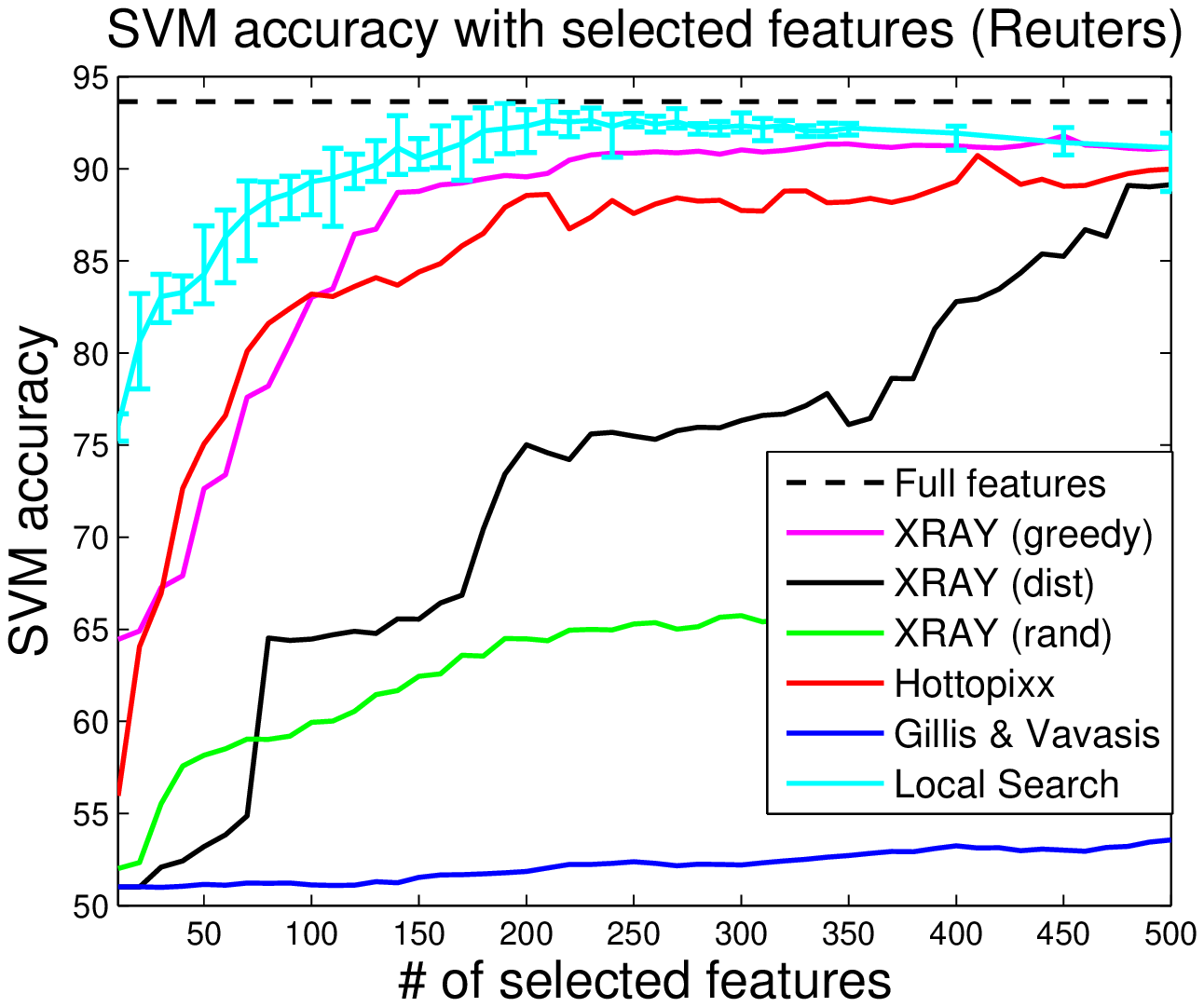}
\end{minipage}
\vspace{-2mm}
\caption{Classification Accuracy using selected features on TDT, BBC and Reuters datasets (best viewed in color)} 
\label{fig:svmacc}
\end{center}
\vspace{-3mm}
\end{figure*}

\begin{figure*}[ht]
\begin{center}
\hspace{-5mm}\begin{minipage}{0.33\textwidth}
\centering
\includegraphics[width=1.1\columnwidth, height=4.5cm]{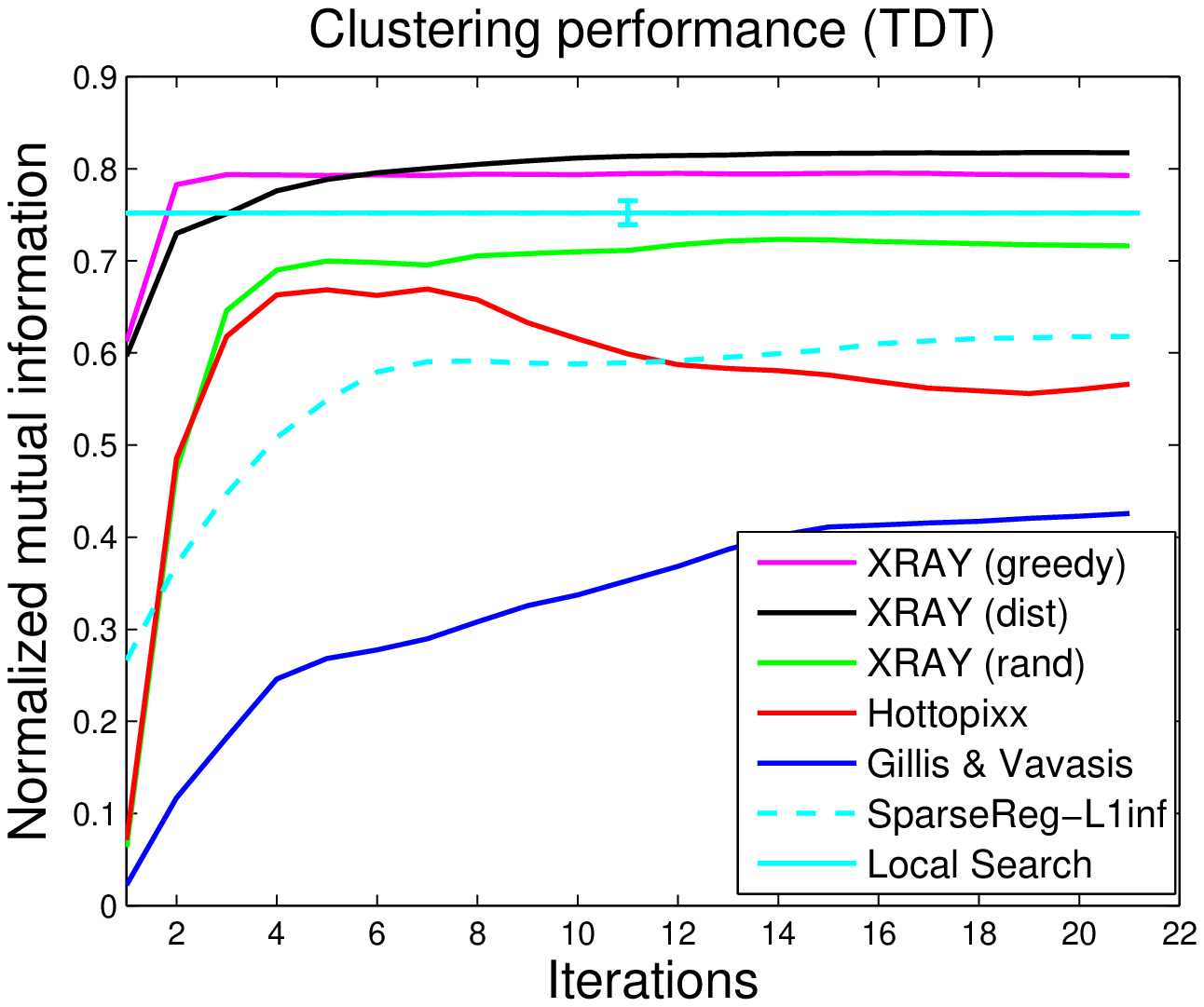}
\end{minipage}
\begin{minipage}{0.33\textwidth}
\centering
\includegraphics[width=1.1\columnwidth, height=4.5cm]{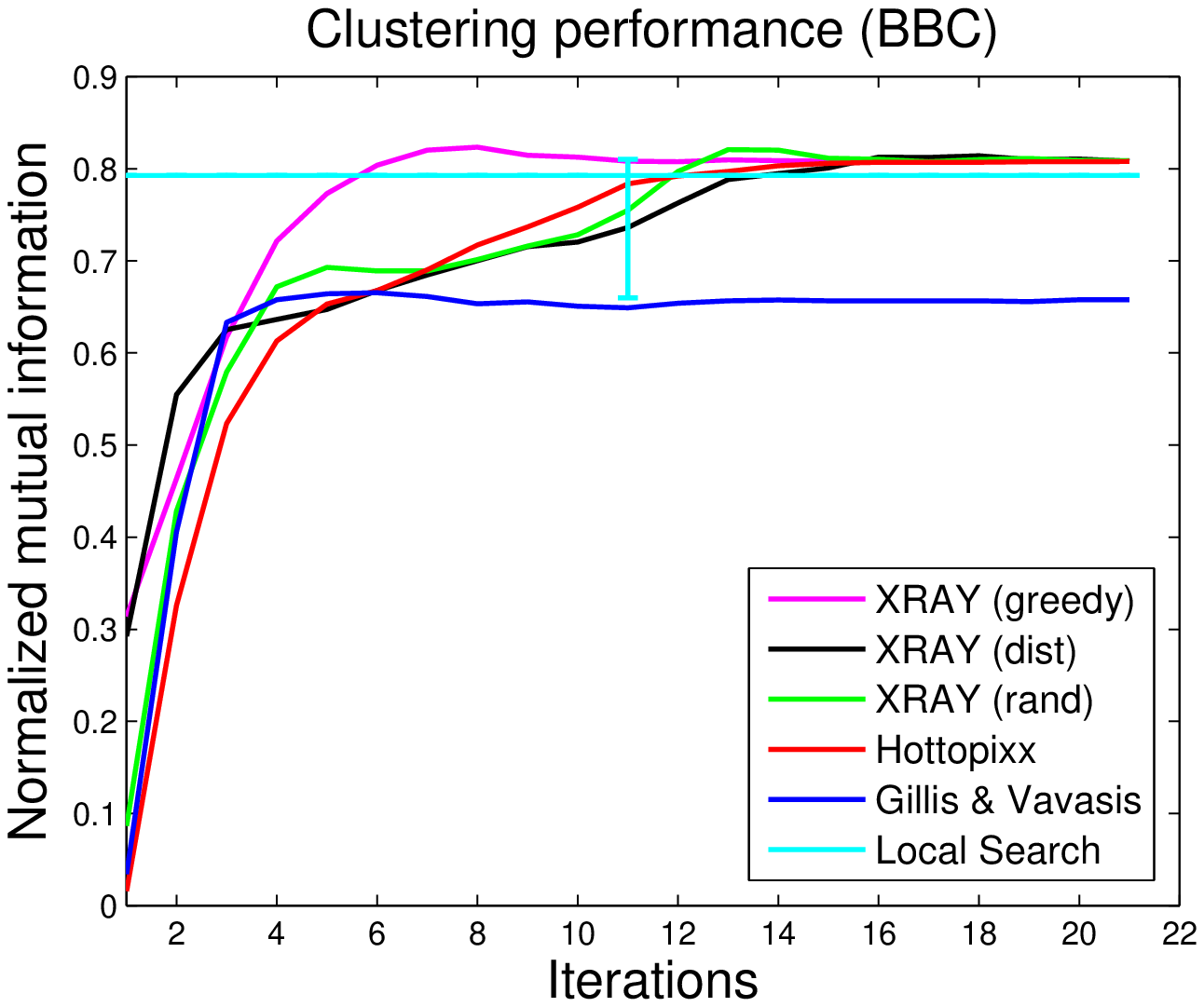}
\end{minipage}
\begin{minipage}{0.33\textwidth}
\centering
\includegraphics[width=1.1\columnwidth, height=4.5cm]{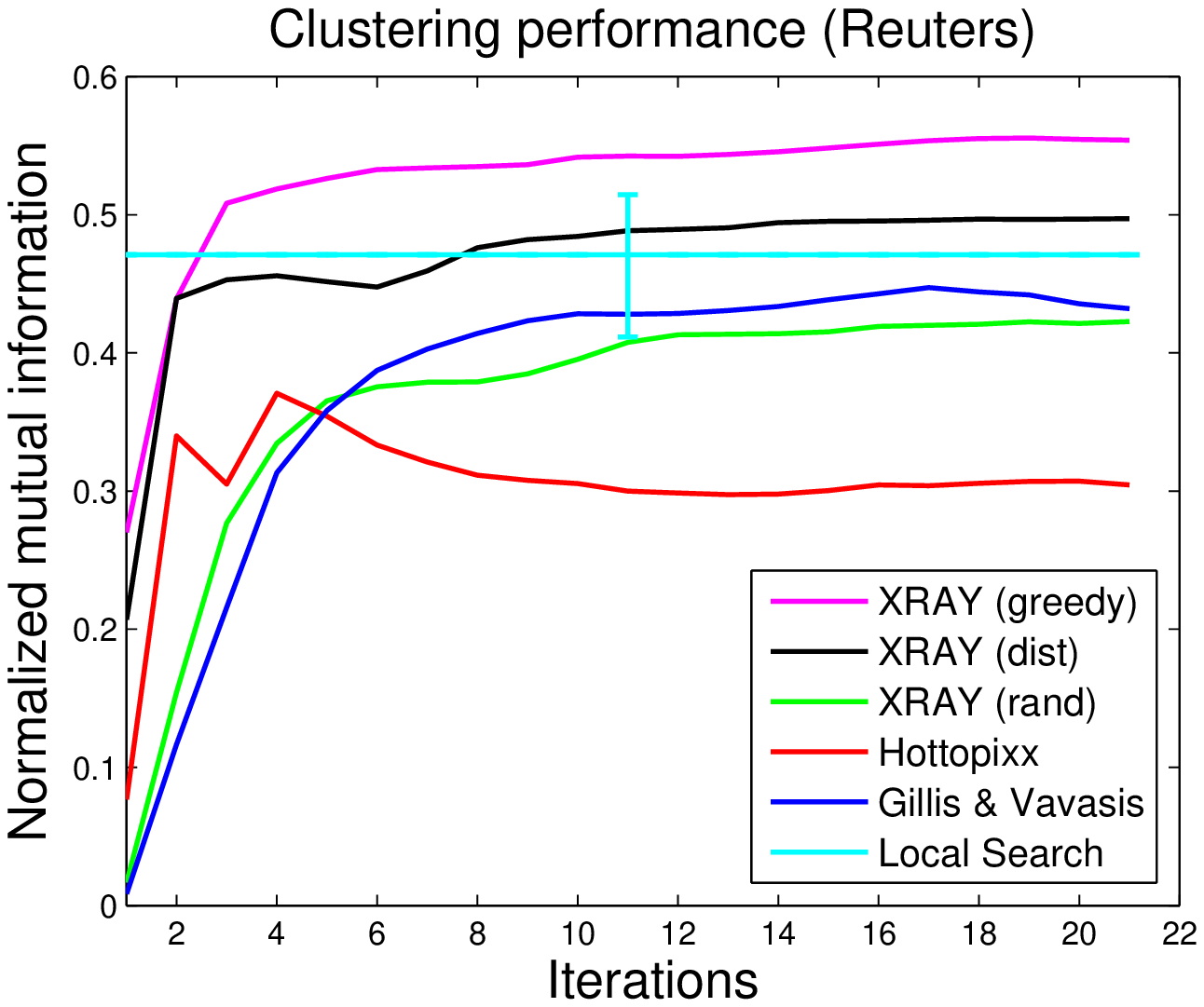}
\end{minipage}
\vspace{-2mm}
\caption{Clustering Performance on TDT, BBC and Reuters datasets (best viewed in color)} 
\label{fig:nmi}
\end{center}
\vspace{-5mm}
\end{figure*}

{\bf Classification experiments:} Figure~\ref{fig:svmacc} shows the classification accuracy results obtained with the features (columns of the document-term matrix restricted to anchor words) selected by different methods on the three datasets. 
\remove{The number of selected words (inner-dimension
$r$ of the factorization) is varied from $k$ to $500$ for each dataset, where $k$ is the number of unique labels
(30 for TDT, 5 for BBC and 10 for Reuters). We also report the classification accuracy with $\vW$ matrix
learned using local optimization methods for inner dimension $r$ varying from $k$ to $500$. We do ten runs of the local
search method for every value of $r$ with different random initializations and plot the 
mean accuracy (the error bars show the variation
around the mean accuracy).} Black dotted line is the classification accuracy with full features (all the words).
We use $5\%$ of the documents for training and the rest $95\%$ for
testing to emulate a semi-supervised learning scenario where we view various methods as inducing a topical representation based on all (unlabeled) data. We use multiclass SVM classifier as implemented 
in LIBLINEAR \citep{liblinear} 
and use four-fold cross validation to select the parameter $C$.
%
{\it Among separable NMF techniques, the proposed \xray (greedy) and \xray (dist) (with exception on Reuters) outperform \hott and GV
on all the three datasets, more so on TDT}. On average, traditional NMFs with local optimization perform quite well on these datasets especially when $r$ is small, but can show significant performance variance (shown as error bars) with respect to initialization.  As the number of topics increases, the performance gap between
the proposed methods and the local optimization method rapidly diminishes. {\it In this regime our techniques
are a viable alternative to local optimization methods, and have the advantage of being local-minima-free, i.e., eliminating uncertainty with respect to initialization
and therefore not requiring multiple runs}. \\

{\bf Clustering experiments}: We also evaluated clustering performance by assigning a cluster label to each document based on the maximum element in the corresponding row of $\vW$. 
We refine the solution with a few iterations of alternating optimization. Figure~\ref{fig:nmi} shows the clustering performance in terms of Normalized Mutual Information (NMI) as these iterations proceed. We also show
the NMI obtained with local search method after it has converged to a local optimum 
(averaged NMI from ten runs with different random initializations
is shown; error-bar indicates the variation around the average). {\it Again, the proposed \xray methods are among the best performing
methods in terms of clustering performance and do not require multiple runs as traditional NMFs do}.  \\
\remove{ 
$\vW^{(0)}=\vW$ using a sequence of nonnegative least squares iterations. 
Each iteration 
consists of two steps in order: $\vH^{(i)}=\argmin_{\vH\geq 0} \lVert \vX - \vW^{(i-1)}\vH\rVert_F^2$ 
and $\vW^{(i)} = \argmin_{\vW\geq 0} \lVert \vX - \vW\vH^{(i)}\rVert_F^2$. NMI at iteration $i$ 
is computed using the matrix $\vW^{(i)}$. This process is similar to running local search methods for
NMF~\cite{cjlin.nmf07,ndho.descentnmf} with $\vW$ initialized from the separable NMF based methods.
}

{\bf Effect of column normalization:} 
In text processing, tf-idf features are popular due to their good
empirical performance in various tasks. 
However, most of the previously proposed methods for the separable NMF problem
\citep{bittorf.12,gillis.12,arora.stoc12} require the columns of $\vX$ 
(i.e, words for text data) to be
$\ell_1$ normalized, which can disturb the tf-idf structure. 
We conduct a small experiment to study the effect of word normalization on the prediction performance
of the proposed methods. 
We identify anchor sets $A^{(\ell_1)}$, $A^{(\ell_2)}$ and $A$ by the proposed methods using
the data matrices $\vX^{(\ell_1)}$ ($\ell_1$-normalized columns), $\vX^{(\ell_2)}$ ($\ell_2$-normalized columns)
and $\vX$, respectively and use $\vX_{A^{(\ell_1)}}$, $\vX_{A^{(\ell_2)}}$ and $\vX_{A}$ for classification (same SVM setup
as described earlier).
Table~\ref{tab:svmaccdiffnorms} shows the classification accuracy for 100 topics on the three datasets. 
These empirical results suggest that $\ell_1$ {\it normalization
of words on top of tf-idf features, as required by other separable NMF methods, can actually adversely affect the predictive quality of the selected anchors}.
\begin{table*}[h]
\centering
{
\caption{Effect of word normalization. The numbers are classification accuracies with selected features for $r=100$.}
\begin{tabular}{c|c|c|c|c|c|c|}
\cline{2-7}
 		& \multicolumn{3}{c|}{\xray (dist)} & \multicolumn{3}{c|}{\xray (greedy)} \\ \cline{2-7}
 		& $\ell_1$ & $\ell_2$ & None & $\ell_1$ & $\ell_2$ & None \\ \hline
\multicolumn{1}{|c|}{TDT} & 21.06 & 83.04 & 84.52 & 31.69 & 90.05 & 91.87 \\ \hline 
\multicolumn{1}{|c|}{BBC} & 58.59 & 84.36 & 87.73 & 75.41 & 82.18 & 87.82\\ \hline 
\multicolumn{1}{|c|}{REUT} & 51.88 & 76.88 & 64.46 & 55.65 & 81.28 & 83.02 \\ \hline 
\end{tabular}
\label{tab:svmaccdiffnorms}
}
\end{table*}

 \begin{table}[h]
\centering
{\small
\caption{Anchors (red) and top keywords for a few sample topics from TDT dataset.}
\begin{tabular}{|l|l|}
\hline
\multirow{3}{*}{\small \xray (dist)} & \red{lewinsky};monica;grand;jury;starr;intern;white;clinton;house;counsel \\
  & \red{iraq};weapons;baghdad;inspectors;iraqi;annan;council;military;inspections;sites\\
  & \red{tobacco};industry;senate;settlement;smoking;bill;legislation;companies;billion;minnesota \\
  & \red{suharto};indonesia;indonesian;habibie;jakarta;president;riots;anti;reforms;resign \\
  & \red{shuttle};columbia;space;astronauts;nasa;mission;;crew;rats;aboard;experiments\\ \hline

\multirow{3}{*}{\small \xray (greedy)} & \red{lewinsky};monica;grand;jury;starr;intern;white;clinton;house;counsel \\
  & \red{iraq};weapons;baghdad;inspectors;inspection;military;council;united;team;strike \\
  & \red{tobacco};industry;senate;settlement;smoking;bill;legislation;companies;billion;minnesota \\
  & \red{suharto};indonesia;habibie;indonesian;jakarta;president;riots;anti;resign;reforms \\
  & \red{shuttle};columbia;space;astronauts;nasa;mission;crew;rats;aboard;experiments\\ \hline
  
\multirow{3}{*}{\small Gillis \& Vavasis} & \red{acknowledging};constitute;contact;physical;privately;intern;sexual;relationship;matter\\ 
  & \red{approves};warns;word;baghdad;deal;nations;iraq;united;approved;financing\\ 
  & \red{successive};alive;votes;checking;bill;supporters;dead;tobacco;stories;senate \\
  & \red{grows};secure;feels;worst;suharto;leave;hour;indonesia;country;crisis \\
  & \red{tall};astronauts;columbia;space;rotating;experiments;ear;inch;chair;stretch\\\hline 
  
\multirow{3}{*}{\small \hott} & \red{acknowledging};constitute;physical;intern;sexual\\  
  & \red{ability};weapons;iraq;destruction;deny;tuned;develop;determined;mass;clinton \\
  & \red{accountable};tobacco;fighting;brands;bill;gop;legislation;desperately;intend;senate \\
  & \red{absence};worsened;turmoil;indonesia;political;suharto;leads;thinks;exercise;track \\
  & \red{approaching};shuttle;columbia;space;nasa;experiments;extending;astronauts;weather\\ \hline
\end{tabular}
\label{tab:topicstdt}
}
\end{table}

\begin{table*}[t]
\centering
{\small 
\caption{Datasets used for large-scale experiments. Times are for $r=100$ on
eight cores on \texttt{daniel}.} 
\begin{center}
\begin{tabular}{|c|c|c|c|c|c|c|c|} 
\hline
Name & $\#$documents & $\#$words & nnz($\vX$) & nnz($\vX^T\vX$) & Sparsity($\vX^T\vX$) & Time($r=100$) & Memory \\
\hline
RCV1 & 781265 & 43001 & 59.16e06 & 172.3e06  & 5\% & 409 secs & 3.6 GB\\
\hline
PPL2 & 351849 & 44739 & 19.43e06 & 1.99e09 & 99.8\% & 1147 secs & 30.2 GB\\
\hline
IBMT & 124708 & 25998 & 1.03e06 & 1.77e06 & 0.2\% & 9.8 secs & 1 GB\\
\hline
\end{tabular}
\end{center}
\label{tbl:datasets}
}
\end{table*}
\begin{table*}[t]
\centering
{
\caption{Running times of \xray versus \hott on $8$ threads for $r$ topics and $E$ epochs.}  
\begin{tabular}{|c|c|c|c||c|c|c||c|c|c|}  \hline
\multirow{2}{*}{Dataset} & 
\multicolumn{3}{|c||}{\xray (secs)} &
\multicolumn{6}{|c|}{\hott (secs)} \\ \cline{2-10}
 & & & & \multicolumn{3}{|c||}{E=5} & \multicolumn{3}{|c|}{E=10} \\ \cline{5-10}
 & $r$=25 & $r$=50  & $r$=100 & $r$=25 & $r$=50 & $r$=100 & $r$=25 & $r$=50 &$r$=100 \\ \hline
IBMT&0.38&1.78 &9.8 &338.6 & 337.2& 327.7&642.1 &668.5 & 636.9\\ \hline
RCV1&15.4&67.2 &409 &2026.8&1938.3&1883.6&3769  &3774.7&3888.9\\ \hline
PPL2&196 &443.8&1147&1818.1&1935.5&1892.8&3725.2&3895.5&3913.7\\ \hline
\end{tabular}
}
\label{tbl:comp}
\end{table*}

{\bf Quality of anchor words}: Qualitatively, we found that anchor words selected by the proposed \xray methods tend to be more representative of the topics compared to those
selected by \hott and GV.  Table~\ref{tab:topicstdt} shows top words and anchors for a few topics 
(\emph{Lewinsky scandal}, \emph{Iraq nuclear program}, \emph{National Tobacco Settlement}, \emph{Indonesia riots of 1998} and \emph{Columbia space shuttle})
extracted from the TDT dataset.


\subsection{Large-scale Experiments}
\label{subsec:implementation}
\begin{figure}[t]
\centering
\includegraphics[width=0.4\linewidth,height=6cm]
{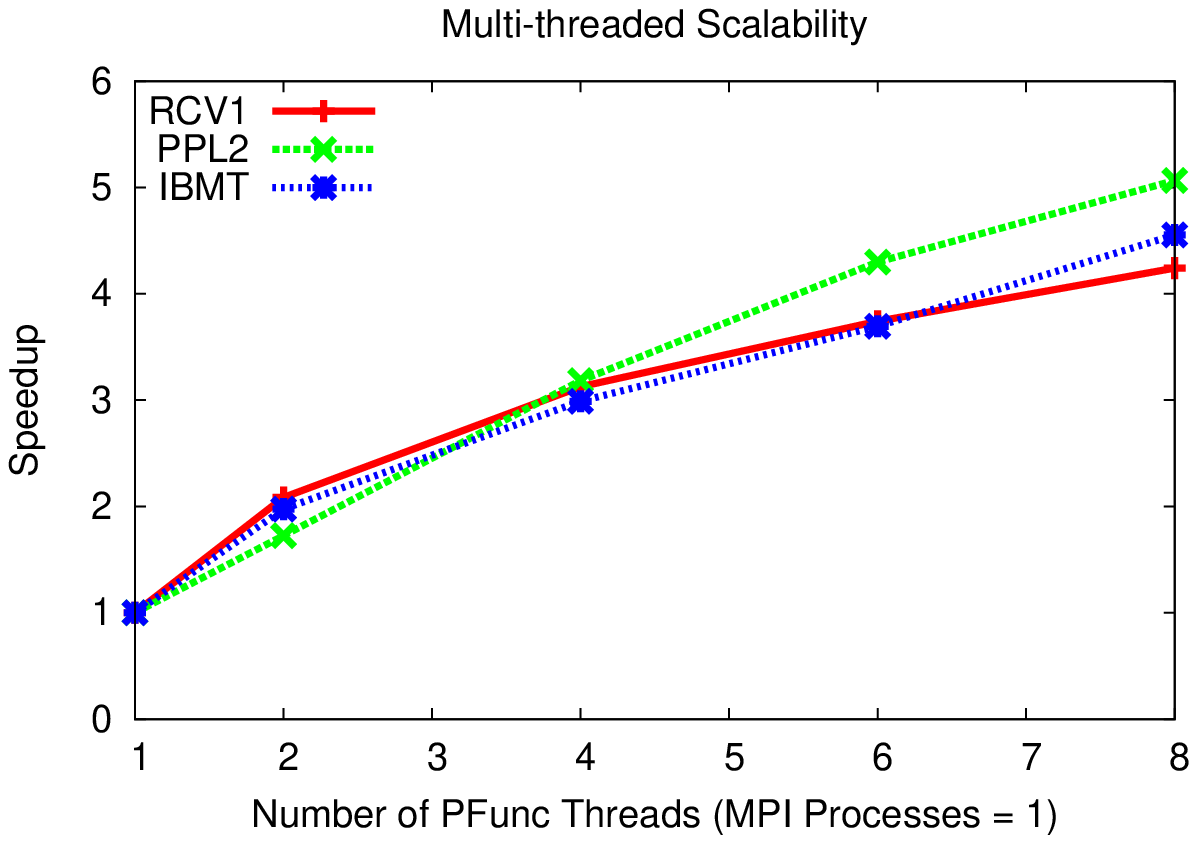}
\caption{Multi-core speedup of RCV1, IBMT, and PPL2 datasets when running on
\texttt{daniel}. All times are for R=100.}
\label{fig:pfunc}
\vspace{-4mm}
\end{figure}

We implemented a shared- and distributed-memory parallel version of
\xray in \Cpp{}.
That is, our implementation can exploit parallelism when running on multi-core
machines, or on clusters of multi-core machines.
For shared-memory parallelism, we use PFunc~\citep{kambadur09:pfunc}, a
lightweight and portable library that provides C and \Cpp{} APIs to express
task parallelism.
For distributed-memory parallelism, we use MPI\footnote{\url{http://www.mpi-forum.org/}}, a popular
library specification for message-passing that is used extensively in
high-performance computing.

To test the shared-memory performance and scalability of \xray, we ran
experiments on \texttt{daniel}, a dual-socket, quad-core Intel\textregistered{}
Xeon\texttrademark{} X5570 machine with 64GB of RAM running Linux Kernel
2.6.35-24 (total 8 cores).
For compilation, we used GCC v4.4.5 with: ``\texttt{-O3 -fomit-frame- pointer
-funroll-loops}'' in addition to PFunc 1.02, OpenMPI 1.4.5 and untuned ATLAS
BLAS.
We ran large-scale experiments on three datasets: RCV1 \citep{RCV1}, 
co-occurence matrix of people and places from ClueWeb09 
dataset \citep{PPL2}, and IBM Twitter (IBMT) dataset. 
The statistics relating to these three large datasets are presented in 
Table~\ref{tbl:datasets}.We report scalability results for \xray (greedy) - other variants are computationally very similar.
%
 
%
Figure~\ref{fig:pfunc} depicts the multi-threaded performance of our
implementation on \texttt{daniel} while detecting 100 topics.
Our implementation is able to factorize RCV1 in 409 seconds on 8 cores and
achieve $4.2x$ speedup over 8 threads when compared to the sequential
implementation.
Similarly, for IBMT we achieve $4.5x$ speedup, while completing the
factorization in $9.8$ seconds on 8 cores. For the dense $\vX^T \vX$ case, we are able to factorize PPL2 in 1147 seconds with just 8
cores.
We believe that further speedup improvements can be demonstrated on these problems  
by (a) optimizing the data layout of various sparse matrices to
alleviate memory contention amongst threads, and (b) in dense problems such as
PPL2, by using a version of BLAS tuned to our architecture and by reorganizing
our implementation around more BLAS-3 operations that have better memory to
compute ratio than BLAS-1 or 2 operations.
Our implementation showed good scalability on
distributed-memory machines as well (details omitted for brevity).

To compare our performance against the state-of-the-art \hott algorithm \citep{bittorf.12}, we ran
their algorithm on \texttt{daniel} with the options ``\texttt{--dual 0.01
--epochs 10 --splits 8 --hott <R> --normse 1 --primal 1e-6}'' set in close consultation
with the authors.
A detailed comparison is shown in Table~\ref{tbl:comp}.
A head-to-head comparison is difficult because of the different performance
characteristics of \hott and \xray. 
For example, \hott's {\it per-epoch} runtime is not dependent on $r$, the number
of topics, but it's accuracy is dependent on $E$, the number of epochs, while
our methods execute exactly $r$ iterations, where each iteration has a
superlinear dependence on $r$. 
Nonetheless, for all three datasets, we see that \xray performs better
than \hott even when \hott is run only for 5 epochs.
In particular, for the sparse datasets IBMT and RCV1, \xray runs to 
completion in significantly shorter amount of time than \hott.
%

\section{Conclusions and Future Work}
Our  methods perform favorably in comparison to other recently proposed separable NMF algorithms and offer highly scalable local-minima-free alternatives to existing 
local optimization techniques. Future work includes a formal noise analysis of the proposed algorithms, investigating the streaming
setting where documents or words arrive in an online fashion, and using our models for social media content analysis. \\
 

{\bf Acknowledgments:} We thank Victor Bittorf and Ben Recht for graciously providing their code, associated parameters and technical support. We thank Haim Avron, Christos Boutsidis, Ken Clarkson, Rick Lawrence and Ankur Moitra for insightful and enthusiastic discussions. \\

{\textit{Research was sponsored by the U.S. Defense Advanced Research Projects Agency (DARPA) under the Social Media in Strategic Communication (SMISC) program, Agreement Number W911NF-12-C-0028. The views and conclusions contained in this document are those of the author(s) and should not be interpreted as representing the official policies, either expressed or implied, of the U.S. Defense Advanced Research Projects Agency or the U.S. Government. The U.S. Government is authorized to reproduce and distribute reprints for Government purposes notwithstanding any copyright notation hereon.}}


\sloppy
\bibliography{fastconicalhull}
\bibliographystyle{icml2012}

\end{document}